\documentclass[letterpaper, 10 pt, journal, twoside]{IEEEtran}
\IEEEoverridecommandlockouts                            

\usepackage{graphicx}
\usepackage{subfigure}
\graphicspath{{figures/}{./}}
\usepackage[utf8]{inputenc}
\usepackage[T1]{fontenc}
\usepackage{fancyhdr}    
\usepackage{textcomp}  
\usepackage{amsmath} 
\usepackage{amssymb}
\usepackage{dsfont}
\usepackage{psfrag}
\usepackage{stfloats}
\usepackage{color}
\definecolor{BlueGreen}{RGB}{0, 158, 115}
\definecolor{Vermillion}{RGB}{213, 94, 0}
\usepackage{multicol}
\usepackage{algorithmic}
\usepackage{algorithm}
\usepackage{array}
\usepackage{bm}
\usepackage{bbm}
\usepackage{paralist}{}
\usepackage{wasysym}
\usepackage[normalem]{ulem}

\usepackage[bookmarks=true]{hyperref}

\usepackage{amsthm}
\newtheorem{theorem}{Theorem}
\newtheorem{lemma}{Lemma}

\newtheorem{problem}{Problem}
\newtheorem{definition}{Definition}

\usepackage{comment}

\usepackage{biblatex}
{\footnotesize \addbibresource{FairRedundantAssignment.bib}}

\newcommand{\A}{\mathcal{A}}
\newcommand{\F}{\mathcal{F}}
\newcommand{\Fo}{\mathcal{F}_\mathcal{O}}
\newcommand{\Ocurl}{\mathcal{O}}
\newcommand{\J}{J_j(\A)}
\newcommand{\Jhat}{\hat{J}_j(\A)}

\begin{document}
\title{Fair Robust Assignment using Redundancy}

\author{Matthew Malencia$^{1,2}$, Vijay Kumar$^{1}$, George Pappas$^{1}$, and Amanda Prorok$^{2}$%
\thanks{Manuscript received: October 15, 2020; Revised: January 11, 2021; Accepted: February 19, 2021.}
\thanks{This paper was recommended for publication by Editor M. Ani Hsieh upon evaluation of the Associate Editor and Reviewers' comments.} 
\thanks{This work was supported by ARL Grant DCIST CRA W911NF-17-2-0181, NSF Grant CNS-1521617,
ARO Grant W911NF-13-1-0350, ONR Grants N00014-20-1-2822 and
ONR grant N00014-20-S-B001, and
Qualcomm Research. The first author acknowledges support from the National Science Foundation Graduate Research Fellowship under Grant No. DGE-1845298.} 
\thanks{$^{1}$The first three authors are with GRASP Laboratory, University of Pennsylvania, Philadelphia, PA 19104-6228. \textit{emails:}
{\tt\footnotesize \{malencia,kumar,pappasg\}@seas.upenn.edu}}%
\thanks{$^{2} $The first and fourth authors are respectively co-affiliated and affiliated with University of Cambridge, William Gates Building, 15 JJ Thomson Avenue, Cambridge CB3 0FD, UK. \textit{emails:}
{\tt\footnotesize \{mpm69,asp45\}@cam.ac.uk}}
\thanks{Digital Object Identifier (DOI): see top of this page.}
}

\markboth{IEEE ROBOTICS AND AUTOMATION LETTERS. PREPRINT VERSION. ACCEPTED MARCH, 2021}
{Malencia \MakeLowercase{\textit{et al.}}: Fair Robust Assignment using Redundancy} 

\maketitle

\begin{abstract}
We study the consideration of fairness in redundant assignment for multi-agent task allocation. It has recently been shown that redundant assignment of agents to tasks provides robustness to uncertainty in task performance. However, the question of how to \textit{fairly} assign these redundant resources across tasks remains unaddressed. In this paper, we present a novel problem formulation for fair redundant task allocation, which we cast as the optimization of worst-case task costs under a cardinality constraint. Solving this problem optimally is NP-hard. We exploit properties of supermodularity to propose a polynomial-time, near-optimal solution. In supermodular redundant assignment, the use of additional agents always improves task costs. Therefore, we provide a solution set that is $\alpha$ times larger than the cardinality constraint. This constraint relaxation enables our approach to achieve a super-optimal cost by using a sub-optimal assignment size. We derive the sub-optimality bound on this cardinality relaxation, $\alpha$. Additionally, we demonstrate that our algorithm performs near-optimally without the cardinality relaxation. We show simulations of redundant assignments of robots to goal nodes on transport networks with uncertain travel times. Empirically, our algorithm outperforms benchmarks, scales to large problems, and provides improvements in both fairness and average utility.
\end{abstract}


\section{Introduction}

\IEEEPARstart{F}{airness} in algorithms has received increasing attention in both research \cite{barocas_fairness_2017, bellamy_ai_2019, verma_fairness_2018} and policy \cite{edwards_slave_2017, shahriari_ieee_2017, wachter_why_2017}. Multi-agent task allocation algorithms distribute resources among \textit{human-centric} tasks, thus requiring the consideration of fairness. Bias in these systems has myriad sources, such as design choices \cite{mittelstadt_ethics_2016}, models and data \cite{mehrabi_survey_2019}, and interpretation of results \cite{cummings_subjectivity_2020}. In this paper, we focus on the fairness of the objective function. We reevaluate the objective function of redundant multi-agent task allocation under the consideration of fairness and propose an algorithm for solving for this fair objective.

Multi-agent task allocation is studied across domains ranging from operations research to robotics \cite{gale_theory_1989, nam_analyzing_2017, yang_algorithm_2017}. Because finding optimal allocations in this combinatorial problem is challenging, most approaches assume deterministic costs. However, real-world costs are uncertain. For example, in assigning agents to spatially located goals, travel times and agent locations are not deterministic \cite{prorok_privacy-preserving_2017}. To account for this uncertainty, previous approaches define risk-based objective functions \cite{zhou_approximation_2018} or use risk constraints \cite{yang_algorithm_2017}. In recent work, Prorok shows that redundant assignment provides robustness against uncertainty and is complementary to existing approaches \cite{prorok_robust_2020}. For example, in \textit{time-sensitive} applications such as rescue scenarios and robots delivering life saving supplies \cite{ackerman_medical_2018}, redundancy improves performance because tasks are completed by the first robot to arrive based on the \textit{first-come first-to-serve} principle \cite{prorok_robust_2020}. 

\begin{figure}
    \centering
    \includegraphics[scale=0.35]{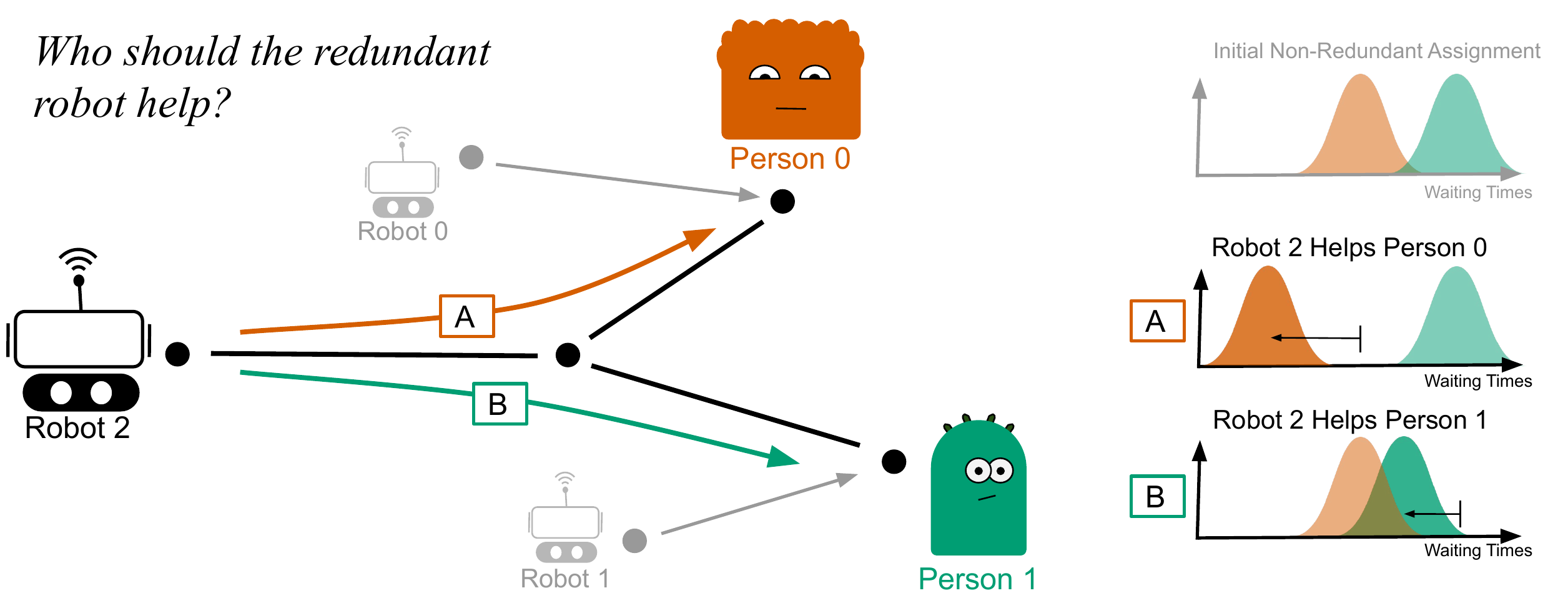}
    \caption{Redundant assignment answers the question, \textit{which task should receive extra resources?} Given an existing non-redundant assignment and respective cost distributions (\textit{top right}), a single redundant robot, Robot 2, can be assigned to either task. Utilitarian approaches assign Robot 2 to Person 0 because of the higher improvement in cost (\textit{plot A}), whereas fair approaches assign Robot 2 to Person 1 because of their higher need (\textit{plot B}). 
    }
    \label{fig:unfair_utilitarianism}
\end{figure}

The redundant assignment problem is classified as Single-Task Robots, Multi-Robot Tasks with Instantaneous Assignment (ST-MR-IA) because each robot can complete all tasks (ST) and multiple robots can be assigned to the same task (MR) \cite{gerkey_formal_2004}. This problem can be cast as a set-partitioning problem which has been proven to be strongly NP-hard \cite{garey_strong_1978}. ST-MR-IA problems require a global optimization function. Traditional methods that optimize utilitarian welfare \cite{harsanyi_bayesian_1978} (the expected sum of costs) provide no guarantees for individual fairness \cite{rawls_theory_2009}. Figure \ref{fig:unfair_utilitarianism} shows that utilitarian approaches seek the best absolute improvement in task cost without accounting for individual task needs.  In \textit{human-centric} applications such as rescue scenarios, delivery systems, and ride-sharing, disparities in task performance yield disparities in the treatment of individual people. This unfairness in utilitarian assignments makes the system unfit with respect to broader social structures and causes a lack of trust in the system. For this reason, it is crucial to consider \emph{fair} optimization objectives. 

Fairness in redundant assignment is an especially important consideration because the tasks served in these allocation problems are often both human-centric and high-risk. For example, in rescue scenarios such as robots delivering life saving supplies, the choice of how robots are allocated impacts how long different individuals must wait for the life saving supplies. This allocation strategy stems directly from the objective function chosen by the system designer.

There exist many definitions of fair objectives \cite{brams_fair_1996, moulin_fair_2004, zhang_fairness_2015}. For example, proportional division seeks assignments that provide agents with equal proportional utility but has no guarantees for indivisible resources \cite{suksompong_asymptotic_2016}, and envy-freeness seeks assignments such that no agent prefers another agent's assignment over its own \cite{arnsperger_envy-freeness_1994}. While fair objectives are often \textit{intra}-personal comparisons, where agents evaluate others' assignments according to their own utility function, experimental work shows that human behavior is better predicted by \textit{inter}-personal comparisons, where agents compare outcomes \cite{herreiner_envy_2009}. For redundant assignment, we seek to minimize the maximum cost among the set of tasks. Specifically, we answer the question: \textit{how can we find a redundant assignment that is fair to the worst-off task?} Hence, we define an inter-task criteria based on John Rawls' \textit{veil of ignorance} \cite{rawls_theory_2009}. 

The `veil of ignorance' thought experiment for the example of emergency supply delivery is as follows. Imagine you get to choose the assignment of delivery robots \textit{and} you will be on the receiving end of this assignment, i.e., one of the humans in need of life-saving supplies. However, you are \textit{ignorant} of your location. Rawls concludes that, under this \textit{veil of ignorance}, the rational conclusion is to choose an assignment that benefits the person with the longest waiting time, since this person could be you. This definition of fairness is often formulated as a minimax or maximin problem \cite{zhang_fairness_2015}. We cast this problem as the optimization of worst-case expected task cost.

\subsection{Related Work}
Noise and uncertainty in real-world applications require managing risk (known unknowns) \cite{borkar_risk-constrained_2014, chaves_risk_2015, majumdar_how_2020}. A common risk measure is Conditional Value at Risk (CVaR), which is the expectation of the tail of a distribution. Wilder \cite{wilder_risk-sensitive_2018} and Zhou et. al \cite{zhou_approximation_2018} use CVaR in approaches to submodular optimization. In addition to risk measures, redundancy provides robustness to uncertainty. Prorok \cite{prorok_robust_2020} shows that redundant assignment improves task performance under uncertainty and is complementary to traditional risk measures. We leverage redundancy in this work to maintain the guarantees for satisfying all constraints of the assignment problem while not precluding the incorporation of traditional risk measures.

Most assignment problems, as well as many problems in robotics, optimize the average expected utility. This default objective is mathematically convenient but has been questioned under considerations such as fairness. Zhang and Shah \cite{zhang_fairness_2015} define four maximin fairness criteria for multi-agent Markov Decision Processes. These definitions of fairness are based on the Rawlsian theory of justice \cite{rawls_theory_2009}. The fairness criteria in this paper similarly applies Rawlsian justice to redundant assignment. Mathematically, this definition is similar to providing robustness to a worst-case objective. Therefore, we look to insights from prior work solving minimax problems.

Minimax problems are difficult to solve because they are highly nonlinear and combinatorial in nature. Uncertainty adds additional complexity to this problem. Because it is NP-hard to solve these problems optimally, approaches exploit mathematical properties of problem formulations to produce polynomial-time, near-optimal solutions. One such property is supermodularity, the property of diminishing returns (see Appendix \ref{supermod}). Supermodular objectives (or in many cases, submodular) can be solved using simple greedy algorithms with a $\frac{1}{2}$ bound on optimality \cite{nemhauser_analysis_1978}. Many objective functions are naturally supermodular, such as information gain \cite{srinivas_information-theoretic_2012} and real-world monetary costs \cite{pigou_laws_1927}. Submodularity has also been exploited to create robustness to worst-case objectives \cite{krause_robust_2008} and arises naturally in redundant task assignments as it captures the attribute of diminishing returns \cite{prorok_robust_2020}.

In a closely related work, Krause \cite{krause_robust_2008} defines the Robust Submodular Observation Selection problem, where the goal is to select an observation set that is robust against the worst-case of multiple objective functions. This work is grounded in the example of sensor placement for monitoring an environmental process. The problem involves finding a set of sensor locations, $\A$, that is a subset of all possible sensor locations, $V$. There exists a set of $\J$ that measure the variance reduction at locations $j$. The objective is to choose $k$ sensor locations that maximize the minimum variance reduction (equivalent to minimizing the maximum variance). This problem is formulated as $ \max _{\A \subseteq V} \quad\min _{j} \J \quad \text{s.t.} \quad |\A| \leq k $ , where each $J_{1}, \ldots, J_{M}$ are normalized monotonic submodular functions.

To solve this problem, Krause \cite{krause_robust_2008} proposes the \textsc{Saturate} algorithm, which finds an approximate solution by relaxing the cardinality constraint, with guarantees on this relaxation. Powers et. al \cite{powers_constrained_2016} extend this work for any matroid constraint. They propose the \textsc{Generalized Saturation} algorithm to solve this problem by relaxing the objective. They guarantee that only a fraction of the set of objective functions exceed a threshold. In contrast to these works, we solve a problem under \textit{both} assignment constraints and a cardinality constraint whereby our algorithm relaxes the constraint on cardinality.

\subsection{Contributions} We propose a novel problem formulation that (for the first time) formalizes the \textit{fair} redundant task allocation problem. Redundancy serves to improve performance under risk. An assumption of supermodularity enables performance guarantees. Notably, we are the first to develop an algorithm for solving a fair redundant assignment problem. Our main contributions are as follows:
\begin{itemize}
    \item We define the Fair Supermodular Redundant Assignment Problem as the minimization of the maximum expected cost across tasks under assignment constraints and cardinality constraints.
    \item We propose an efficient algorithm to solve the Fair Supermodular Redundant Assignment problem approximately by relaxing the cardinality constraint.
    \item We prove a bound on this relaxation and the computational complexity of our approach.
\end{itemize} 
We analyze our algorithm in simulation, showcasing our theory empirically, demonstrating that our algorithm outperforms benchmarks, and analyzing a case study to compare the fairness of our algorithm to prior work.

\section{Preliminaries}
The problem space in this paper is represented by a graph $\mathcal{B}=(\mathcal{U}, \mathcal{F}, \mathcal{C})$. The set of vertices $\mathcal{U}$ is partitioned into two disjoint sets $\mathcal{U}_a$ and $\mathcal{U}_t$, the sets of agent nodes and task nodes, respectively. The edge set $\mathcal{F}=\left\{(i, j) \; | \; i \in \mathcal{U}_{a}, j \in \mathcal{U}_{t}, \;\forall \; i, j\right\}$ defines all possible agent-task pairings available for assignment. For each $(i,j)$ pair in $\mathcal{F}$, there exists a corresponding cost $\mathcal{C}_{i,j}$, which is a random variable. The set of all $\mathcal{C}_{i,j}$ is the cost set, $\mathcal{C}$. While specific instances of this problem can constrain $\mathcal{C}$, we make no assumptions on the form, independence, or correlations of the elements of $\mathcal{C}$.

There are $M$ tasks and $N$ agents, and a total deployment size, $N_d$, specified as an input. Each task $j$ has a corresponding cost function $J_j(\cdot)$ mapping the cost set $\mathcal{C}$ to a real number.\footnote{We later assume that all $J_j(\cdot)$ are supermodular.} For example, $J_j(\cdot)$ could be the sum, minimum, or product of the costs $\mathcal{C}_{i,j}$ for agents $i$ assigned to task $j$. These cost functions represent how a set of redundant agents complete a task; the functions map a subset of the assignment to a real valued cost. In this paper we consider the minimum operator, representing the first-come first-to-serve principle \cite{prorok_robust_2020}.

Fair assignment seeks a pairing of agents to tasks that minimizes the maximum task cost, where agents cannot be assigned multiple tasks, tasks must be assigned at least one agent, and the number of agents is no more than $N_d$. This problem is formulated as finding a subset of the ground set that minimizes the maximum cost: $\min_{\A \subseteq \F} \;  \max_j \J$. This assignment, $\A$, has a constrained cardinality of $|\A| \leq N_d$, and assignment constraints per agent and task. The optimization is formulated with the introduction a second decision variable $\xi$ and constraint of all $\J$ by $\xi$. Finding the minimum $\xi$ such that all $\J \leq \xi$ equivalently minimizes $\max_j \J$. 

\begin{definition}\label{p_space}
\emph{Fair Assignment:}
Given N agents and M tasks with uncertain agent-task assignment costs, find an assignment $\A \subseteq \mathcal{F}$ no larger than $N_d \geq M$ and minimum cost bound $\xi \in \mathbb{R}$, such that no costs $\J$ exceed $\xi$. Formally: 
\begin{equation}\label{full}
\begin{aligned}
\min_{\xi, \A \subseteq \F} \quad &  \xi  \\
\textrm{s.t.} \quad & \J \leq \xi , \quad \forall j \\
 &  |\A| \leq N_d \\
  & \forall \; i, \; |\{j \;|\;(i, j) \in \A \}| \; \leq 1\\
  & \forall \; j, \; |\{i \;|\;(i, j) \in \A \}| \; \geq 1\\
\end{aligned}
\end{equation}
\end{definition}

Fair assignment is a set-partitioning problem and has been shown to be strongly NP-hard \cite{garey_strong_1978}. Approaches that solve similar problems do not have guarantees on all of these constraints \cite{zhou_approximation_2018, zhou_sensor_2019}. To ensure that both assignment constraints are met and to reduce uncertainty through redundancy, we define the Fair Supermodular Redundant Assignment problem.  

\section{Problem Statement}

Consider a system of $M$ tasks and $N$ agents, where each agent-task pairing has random cost, represented by an arbitrary probability distribution.\footnote{Specific instances of this problem can restrict these distributions, but we make no assumptions on their form, independence, or correlations.} For each of the $M$ tasks, a cost function is defined. Our goal is to find an assignment that:
\begin{itemize}
    \item minimizes the maximum task cost,
    \item uses at most $N_d$ agents, 
    \item assigns each agent to at most one task, and
    \item ensures no task is unassigned ($N_d \geq M$).
\end{itemize}

Redundant assignment, where multiple agents can be assigned the same task, provides robustness against uncertainty and improvements to expected performance of tasks \cite{prorok_robust_2020}. If tasks have deterministic costs, this problem simplifies to the linear bottleneck assignment problem which can be solved using the thresholding algorithm \cite{szpankowski_bottleneck_1988}. The presence of uncertainty, however, requires a new approach to this optimization.

This section describes a flexible mathematical formulation of fair assignment under two assumptions: 

\textit{Assumptions}. \textbf{(1)} all cost functions, $J_j(\cdot)$, must be \textit{supermodular},\footnote{$J_j(\cdot)$ can each be different functions but should be similarly scaled so that comparisons among different cost functions are meaningful.} and \textbf{(2)} a prior non-redundant assignment exists.

The first assumption is that all cost functions are supermodular, i.e., they have the property of diminishing return. However, we are solving for the maximum of a set of supermodular functions, which is generally \textit{not supermodular} (see Appendix \ref{supermod}). Therefore, the greedy algorithm has no performance guarantees for this problem. Simple examples, as shown by Krause, prove that the greedy algorithm applied to this problem can perform arbitrarily badly \cite{krause_robust_2008}.

The second assumption requires an initial non-redundant assignment exists, denoted $\Ocurl$, such as one found through standard assignment methods (e.g., threshold algorithm for the linear bottleneck assignment problem). This assumption guarantees that all tasks are assigned at least one agent, thus eliminating this constraint. Additionally, this assumption enables tractability and ensures the supermodularity of cost functions (e.g., unassigned tasks potentially have infinite cost, breaking supermodularity). With this assumption, we focus strictly on assigning the redundant agents. Therefore, the ground set becomes $\F \backslash \Ocurl$, denoted $\F_{\Ocurl}$. Cost functions remain denoted $\J$ though they are a function of $\A$ given $\Ocurl$.

These two assumptions on Definition~\ref{p_space} yield Problem~\ref{FSRA}, the Fair Supermodular Redundant Assignment problem. The constraint in Definition \ref{p_space} that all tasks be completed is guaranteed by the initial assignment (Assumption 2), and therefore omitted. All cost functions, $\J$, are assumed to be supermodular and are a function of $\A$ given $\Ocurl$.

\begin{problem}\label{FSRA}
\emph{Fair Supermodular Redundant Assignment:} Given N agents, M tasks with uncertain agent-task assignment costs, and an initial assignment $\Ocurl$, find an assignment $\A \subseteq \Fo$ and minimum cost bound $\xi \in \mathbb{R}$, such that no cost $\J$ exceeds $\xi$, assuming all functions $\J$ are supermodular. Formally:

\begin{equation}\label{redundant}
\begin{aligned}
\min_{\xi, \A \subseteq \Fo} \quad &  \xi  \\
\textrm{s.t.} \quad & \J \leq \xi , \quad \forall j \\
  & |\A| \leq N_d - M \\
  & \forall \; i, \; |\{j \;|\;(i, j) \in \A \cup \Ocurl \}| \; \leq 1
\end{aligned}
\end{equation}
\end{problem}

This formulation is general and flexible enough to be applied to many assignment problems, as long as the cost functions are supermodular. One such example is the assignment of agents with uncertain travel times to spatially located tasks. In this example, the cost at each node is defined as the expected waiting time, $\J = \underset{c}{\mathbb{E}} \left[ \min_i \left\{ C_{i j } |(i, j) \in \A \cup \Ocurl  \right\} \right]$. The minimum operator here represents the first-come first-to-serve principle introduced by Prorok \cite{prorok_robust_2020}, who showed that this function is supermodular. 

\section{Approach} \label{approach}

The algorithm presented in this section provides an approximate solution, $\A_f$, to the Fair Supermodular Redundant Assignment problem. The solution is approximate because we relax the cardinality constraint for tractability purposes: the solution $\A_f$ can be $\alpha$ times larger than the size of the optimal assignment (this section elaborates on the bound on $\alpha$). By relaxing the size of the solution, this assignment is guaranteed to have a maximum cost that is at most equal to the maximum cost of the optimal solution $\A^*$. That is, \textbf{the approximate solution $\A_f$ has a cost that is less than that of $\A^*$ but a size that is larger}. This section outlines the algorithm, a bound on the cardinality relaxation $\alpha$, the limits on the deployment size $N_d$, and the computational complexity of this approach.

\subsection{The Relaxed Fair Supermodular Redundant Assignment Problem}

The basic premise of our approach to solving Problem \ref{FSRA} is to conduct half-interval search on possible values of $\xi$ and to solve a sub-problem to determine which half-interval cannot contain the target solution. Suppose there exists an algorithm that finds the smallest assignment that yields a maximum task cost of at most $\xi$, which is given as an input:

\begin{equation} \label{sub}
\begin{aligned}
\underset{\A \subseteq \Fo}{\text{argmin}} \quad &  |\A|  \\
\textrm{s.t.} \quad & \J \leq \xi , \quad \forall j \\
  & \forall \; i, \; |\{j \;|\;(i, j) \in \A \cup \Ocurl \}| \; \leq 1
\end{aligned}
\end{equation}

This sub-problem returns a candidate solution $\A_s$. If $\A_s$ has a cardinality that is at most $N_d - M$, then $\A_s$ and the given value of $\xi$ are feasible for Problem \ref{FSRA}. If $|\A_s| > N_d - M$, then $\A_s$ and the $\xi$ value are not feasible. Solving this sub-problem determines whether a given value of $\xi$ is feasible. Therefore, half-interval search can find the optimal (minimum feasible) value of $\xi$. Given a starting range of possible $\xi$ values, the midpoint is tested for feasibility. If feasible, the half-interval below the midpoint contains the optimal value and thus becomes the new range of possible $\xi$ values. Otherwise, the higher half-interval becomes the new range. This procedure repeats until convergence, at which point the returned $\A_s$ and $\xi$ are guaranteed to be optimal.

However, Krause proves with Theorem 3 \cite{krause_robust_2008} that unless $P=NP$, there cannot exist any polynomial time approximation algorithm for Problem \ref{FSRA}. In other words, unless $P=NP$, the sub-problem in equation \eqref{sub} cannot be solved in polynomial time. To enable a polynomial time approximation algorithm for Problem \ref{FSRA}, the cardinality constraint is relaxed by $\alpha$. This results in the \textit{relaxed} Fair Supermodular Redundant Assignment problem, with $\alpha > 1$:

\begin{equation}\label{relaxed}
\begin{aligned}
\min_{\xi, \A \subseteq \Fo} \quad &  \xi  \\
\textrm{s.t.} \quad & \J \leq \xi , \quad \forall j \\
  & |\A| \leq \alpha (N_d - M) \\
  & \forall \; i, \; |\{j \;|\;(i, j) \in \A \cup \Ocurl \}| \; \leq 1
\end{aligned}
\end{equation}

This relaxed problem can now be solved with the aforementioned half-interval procedure. Given a value $\xi$, the sub-problem of the relaxed Fair Supermodular Redundant Assignment problem is still equation \eqref{sub}. However, the relaxation on the cardinality constraint changes the update step of the half-interval search. We begin with a range of values $[\xi_{min},\xi_{max}]$, where $\xi_{min}=0$ because costs are strictly positive and $\xi_{max}$ is the worst waiting time after the initial assignment, $\max_j J_j(\emptyset)$. Given the mid-point of this range, the sub-problem returns a candidate solution $\A_s$. Previously, the cardinality of $\A_s$ was compared to $N_d - M$ to determine its feasibility. Because of the cardinality relaxation, the cardinality of $\A_s$ now informs either the \textit{infeasibility} of the original problem or the \textit{feasibility} of the relaxed problem, adjusting the respective bounds:

\begin{itemize}
    \item If $|\A_s| > \alpha (N_d-M)$, then $\xi$ is not a feasible solution to the original problem.\footnote{Or if there does not exist a feasible $\A_s$.} Set $\xi_{min}=\xi$.
    \item If $|\A_s| \leq \alpha (N_d-M)$, then $\xi$ is a feasible solution to the relaxed problem. Set $\xi_{max}=\xi$.
\end{itemize}

To guarantee a solution, the number of agents used after the relaxation, $\alpha(N_d-M)$, must not exceed the number of agents available, $N-M$. Thus, set $N_d \leq \frac{N-M}{\alpha} + M$.

Through this point, we have assumed that there exists an approach to solving the sub-problem defined in equation \eqref{sub} above. We have also assumed that this solution satisfies the relaxed cardinality constraint $\alpha$. The rest of this section describes both an algorithm that solves this sub-problem and the bound on the cardinality relaxation $\alpha$.

\subsection{Algorithm Details}

We seek an algorithm to sub-problem \eqref{sub} that finds the smallest set $\A$ such that all cost functions $\J$ are below a given cost budget $\xi$. We build on the prior work of Prorok on \textit{optimal matching of redundant agents under a cost budget} \cite{prorok_robust_2020}. Leveraging this work requires two challenges to be addressed. First, Prorok uses a single supermodular function constraint while we constrain a set of $M$ supermodular functions. Second, the relaxation of this sub-problem must hold for all iterations of the half-interval search.  

We transform equation~\ref{sub} to have a single supermodular cost constraint. Let $\Jhat$ be the truncated function $\max \{\J,\xi \}$ and form a single constraint of the average of the set of $\Jhat$. This constraint is equivalent to the original set of cost constraints. The average of truncated cost functions less than $\xi$ implies that all cost functions are less than $\xi$. Additionally, truncation and the average function preserves supermodularity. Because the average of the truncated cost functions has a minimal value of $\xi$, the inequality constraint becomes an equality constraint.

\begin{equation} \label{sub_trunc}
\begin{aligned}
\underset{\A \subseteq \Fo}{\text{argmin}} \quad &  |\A|  \\
\textrm{s.t.} \quad & \frac{1}{M}\sum_j \Jhat = \xi \\
  & \forall \; i, \; |\{j \;|\;(i, j) \in \A \cup \Ocurl \}| \; \leq 1
\end{aligned}
\end{equation}
 
Prorok's greedy algorithm \cite{prorok_robust_2020} is shown to solve this problem in polynomial time with a bound (which we call $\alpha$) on the ratio of the solution set size to the optimal solution set size. However, to use this approximation algorithm in the half-interval search, the bound on $\alpha$ must be independent of $\xi$. Theorem \ref{T} below details a $\xi$-independent bound on $\alpha$.

We solve Problem \ref{FSRA} with Algorithm \ref{a1}, Fair Supermodular Redundant Assignment. After initialization, Algorithm \ref{a1} conducts a half-interval search over $\xi$. For each value of $\xi$, Prorok's Greedy Redundant Assignment (GRA) with dynamic programming \cite{prorok_robust_2020} greedily assigns agent-task pairs until either all costs are below $\xi$ or all agents are assigned. The output of GRA is then used to update the range of $\xi$ to continue the half-interval search until convergence. 

\textit{Algorithm description.} The inputs required are $M$ cost functions $J_j$ that are all supermodular, a deployment size $N_d \geq M$, a relaxation constant $\alpha \geq 1$, and an initial non-redundant assignment $\Ocurl$ that can be found using traditional methods such as the thresholding algorithm for the linear bottleneck assignment problem. Lines 1-3 initialize the range of $\xi$ and the output solution $\A_f$. The \textit{while} loop beginning at line 5 is the half-interval search, whose interval bounds are updated in lines 8-13. Given a value $\xi$ (line 6) and the average truncated cost function defined in line 1, the sub-problem in equation~\ref{sub_trunc} is solved (line 7). The solution of this sub-problem, $\A_s$, is then used to update the half-interval search range. If $|\A_s| \leq \alpha(N_d-M)$, $\xi$ is a feasible approximate solution to Problem \ref{FSRA}. $\xi_{max}$ is updated and $\A_s$ is kept as the best candidate solution, $\A_f$. If $|\A_s| > \alpha(N_d-M)$, there is no feasible solution to Problem \ref{FSRA} for the given value of $\xi$, and $\xi_{min}$ is updated. For every iteration of the half-interval search, $\xi_{min}$ is guaranteed to be infeasible, i.e., lower than minimum $\xi$ in Problem \ref{FSRA}. Therefore, the half-interval search converges to $\xi \leq \xi^*$ and the returned solution $\A_f$ is guaranteed to satisfy $|\A_f| \leq \alpha(N_d-M)$.

Theorem~\ref{T} states that the Fair Supermodular Redundant Assignment algorithm returns a solution $\A_f$ with a cost that is at worst equivalent to the cost achieved by $\A^*$ but a solution size that is up to $\alpha$ time larger, as defined in equation~\eqref{alpha}.
The algorithm works for any value of $\alpha \geq 1$, but values differing from that defined in equation~\eqref{alpha} will not have the guarantees stated in Theorem~\ref{T}.

\begin{theorem}\label{T}
The Fair Supermodular Redundant Assignment algorithm guarantees that
$$ \max_j J_j(\A_f) \leq \min_{|\A| \leq N_d} \max_j J_j(\A) \quad \text{and} \quad
|\A_f| \leq \alpha N_d $$
for a given $N_d$ and $\alpha = 1 + \log (\max_j J_j(\emptyset) - 1)$. The number of function evaluations is
$$ O \Big( (1+ \log_2(M \max_j J_j(\emptyset)) \; (N_d-M)NMS \Big)$$
where $S$ is the number of samples taken from the cost distributions.
\end{theorem}

\begin{proof}
Let $\Bar{J}(\A) = \frac{1}{M}\sum_j \Jhat$. Prorok's Proposition 3 \cite{prorok_robust_2020} shows that for our sub-problem, $\alpha = 1+\log \frac{\Bar{J}(\emptyset) - \xi}{ \Bar{J}\left(\mathcal{A}_{k-1}\right) - \xi}$, where $\mathcal{A}_{k-1}$ is the set chosen by the greedy algorithm in its penultimate step (i.e., the next assignment by the greedy algorithm achieves $\xi$). Because $\Bar{J}\left(\mathcal{A}_{k-1}\right)$ can be arbitrarily close to $\xi$, the bound can tend towards infinity, preventing a $\xi$-independent bound on $\alpha$. Assuming $\J$ only takes integer values, though, limits $\Bar{J}\left(\mathcal{A}_{k-1}\right) - \xi$ to minimum value of one, and ensures $\xi \geq 1$ because costs are strictly positive, thus yielding a $\xi$-independent bound of
\begin{equation} \label{alpha}
\alpha = 1 + \log (\max_j J_j(\emptyset) - 1) \geq 1 + \log (\Bar{J}(\emptyset) - \xi) 
\end{equation}

This result is similar to that derived by Wolsey \cite{wolsey_analysis_1982}. Additionally, Krause builds on Wolsey's result to show in Section~7.1 \cite{krause_robust_2008} that integer objective functions can be easily extended to take rational values by rounding by their highest order bits, allowing a small additive approximation error. Simulations in Section~\ref{exp} show near-optimal performance without the need for this rounding.

The computational complexity comes from the number of times the half-interval search runs $(1+ \log_2(M \max_j J_j(\emptyset))$ \cite{krause_robust_2008} and the $(N_d-M)NMS$ evaluations in the inner loop sub-problem \cite{prorok_robust_2020}.
\end{proof}

\begin{algorithm}
\caption{Fair Supermodular Redundant Assignment}
\begin{algorithmic}[1] \label{a1}
    \REQUIRE $J_1, ... , J_M, N_d, \alpha, \Ocurl$
    \ENSURE Set of fair redundant assignments, $\A_f$
    \STATE Define $\Bar{J}(\A,\xi) =  \frac{1}{M}\sum_j \max\{\J,\xi\}$
    \STATE $ \xi_{min} \xleftarrow{}  0 $
    \STATE $ \xi_{max} \xleftarrow{}  \max_j J_j(\emptyset) $ 
    \STATE $\A_f \xleftarrow{} \emptyset $ 
    \WHILE{$\xi_{max} - \xi_{min} \geq \frac{1}{M}$}
        \STATE $ \xi \xleftarrow{}  \frac{\xi_{min} + \xi_{max}}{2}$
        \STATE $\A_s \xleftarrow{} \mathrm{GRA}(\Bar{J}, \xi, \Ocurl)$
        \IF{$|\A_s| \leq \alpha(N_d - M) $}
            \STATE $\xi_{max} \xleftarrow{} \xi$
            \STATE $\A_f \xleftarrow{} \A_s $
        \ELSE
            \STATE $\xi_{min} \xleftarrow{} \xi$
        \ENDIF
    \ENDWHILE
    \RETURN $\A_f$
\end{algorithmic}
\end{algorithm}

\section{Evaluation}\label{exp}

We present Algorithm \ref{a1} as it applies the case study of robots with uncertain travel times assigned to spatially located tasks, such as in emergency supply delivery. The cost at each node is defined as the expected waiting time, $\J = \underset{c}{\mathbb{E}} \left[ \min_i \left\{ C_{i j } |(i, j) \in \A \cup \Ocurl  \right\} \right]$. The \textit{minimum} operator here represents the \textit{first-come first-to-serve} principle \cite{prorok_robust_2020}. In other words, all robots have the emergency supplies (e.g., water or medicine) to delivery to any task, and whichever robot arrives to a task location first provides the supply thus accomplishing the task. This case study seeks to minimize expected waiting time of the worst task while ensuring the total deployment size (initial plus redundant assignment) is less than $N_d$. Our case study includes three sets of simulations. 

The first two sets of simulations use random bipartite graphs to empirically showcase Theorem~\ref{T} and to demonstrate that Algorithm~\ref{a1} outperforms benchmark algorithms. The bipartite graphs in these simulations are abstract representations of any problem that uses the first-come first-to-serve principle. The last set of simulations uses a random transport network to analyze a use case of Algorithm \ref{a1} and highlight the unfairness that can arise from utilitarian assignment.

\subsection{Empirical Display of Theorem \ref{T}}

In this section, we display and visualize Theorem \ref{T} empirically by showcasing that Algorithm~\ref{a1} yields assignments with more-than-optimal cost when using the cardinality relaxation of $\alpha = 1 + \log (\max_j J_j(\emptyset) - 1)$ and near optimal cost when $\alpha = 1$. The former achieves super-optimality by using a deployment size that is larger than the desired deployment size, while the latter respects the desired deployment size but has no theoretical guarantees on performance (though empirically near-optimal in most cases).

\begin{figure}
    \centering
    \text{\footnotesize \quad \quad Cost of Algorithm 1 Compared to Optimal}\par\medskip
    \includegraphics[scale=0.375]{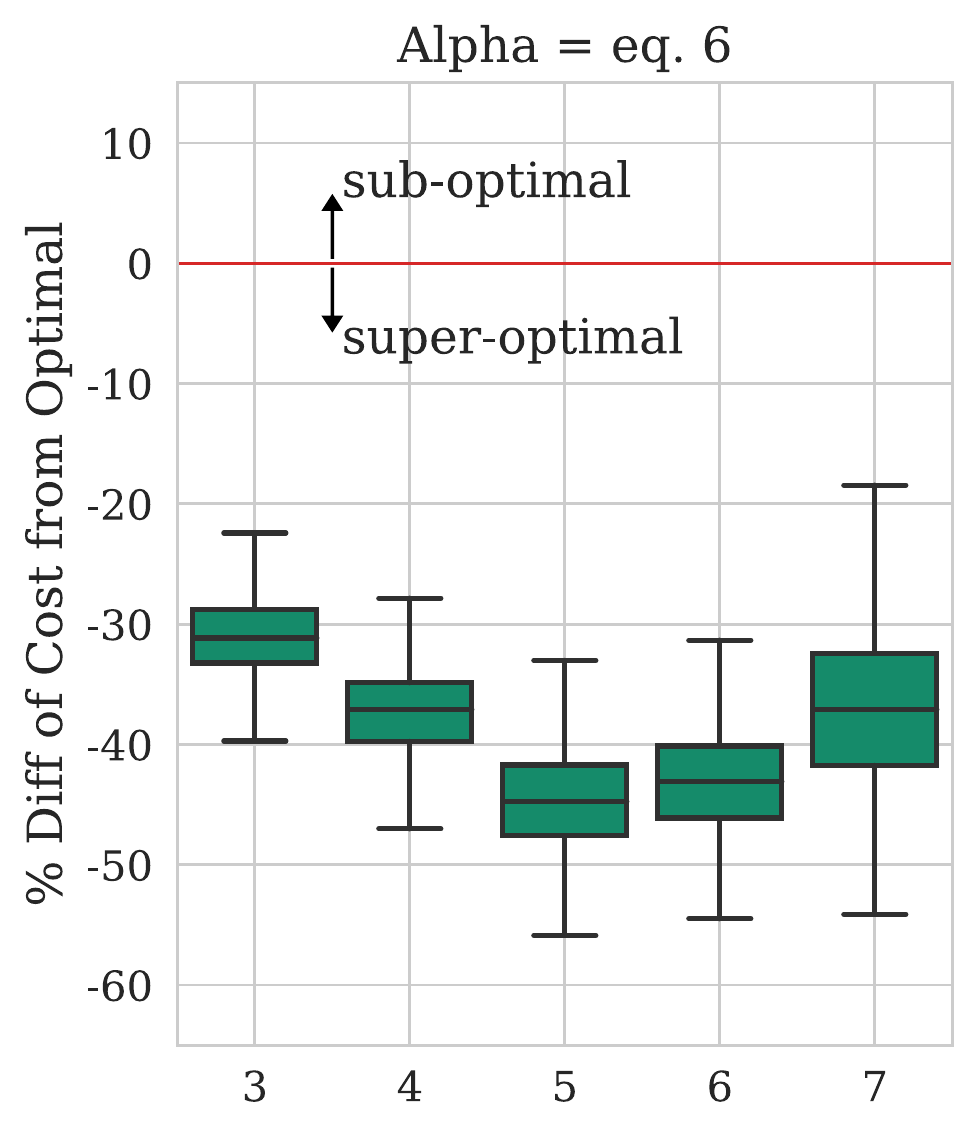}
    \includegraphics[scale=0.375]{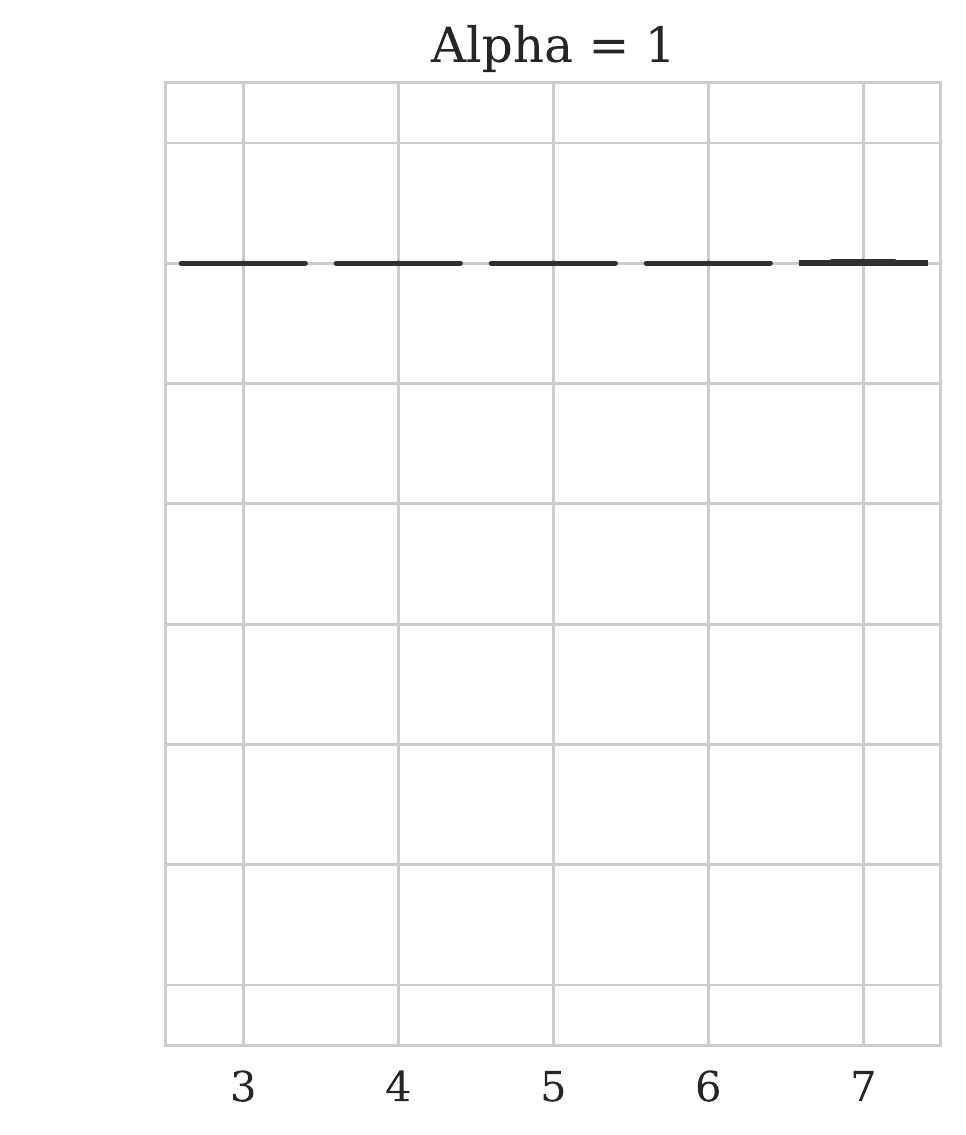}
    \includegraphics[scale=0.375]{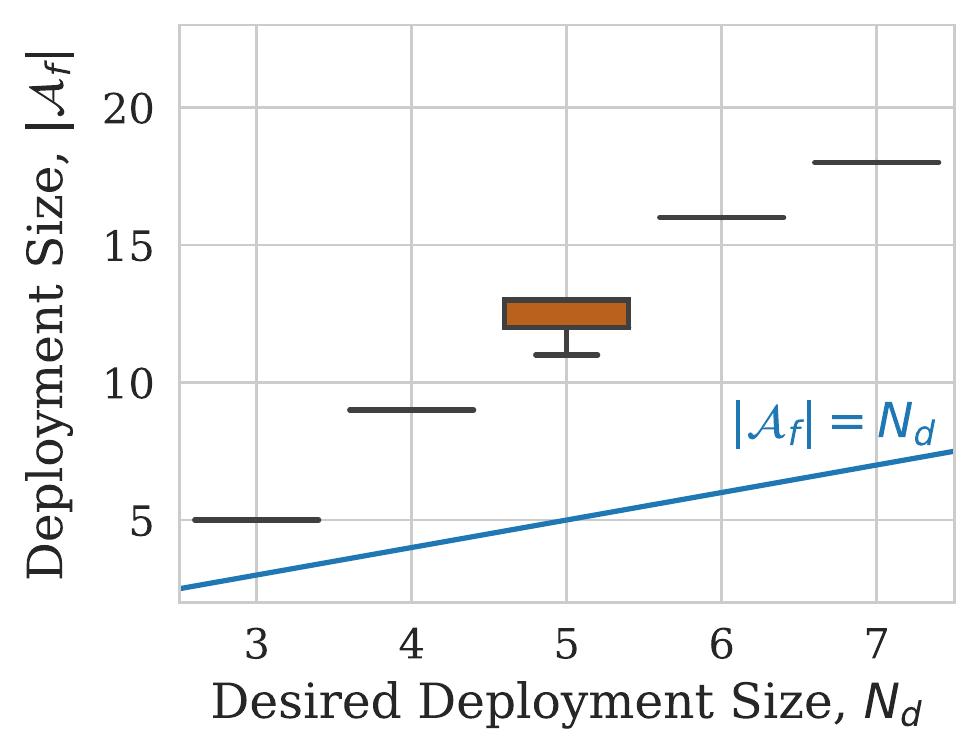}
    \includegraphics[scale=0.375]{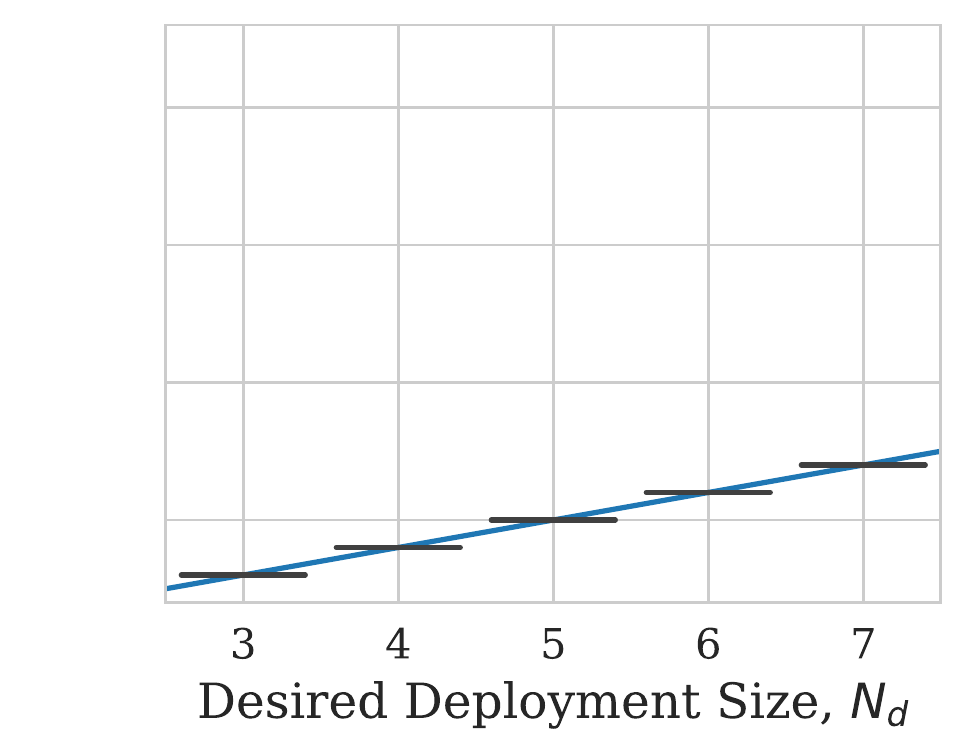}
    \caption{\textit{Top.} The percent difference in cost between the optimal assignment and the assignments produced by Algorithm~\ref{a1}  displays our above proof that the $\alpha$ value proven in Theorem~\ref{T} (\textit{left}) yields super-optimal cost by using larger than optimal assignments,and shows that $\alpha=1$ (\textit{right}) yields near optimal results. Note that \textbf{positive values are sub-optimal} (\textit{top}) because the problem is a minimization. Outliers are omitted. \textit{Bottom.} The comparison of the actual deployment size to the deployment size constraint shows that super-optimality is achieved by relaxing the constraint, yielding larger assignments and that $\alpha=1$ respects the deployment size constraint.
    }
    \label{fig:algo_v_optimal}
\end{figure}

To show these results, Algorithm~\ref{a1} must be compared to the optimal assignment. Finding the optimal is extremely computationally expensive because it requires brute force search over the power set. Therefore, this simulation set is constrained to small examples. We represent a problem with 18 agents, 2 tasks, and deployment sizes $[3, 4, 5, 6, 7]$ as a random bipartite graph, $\mathcal{B}=(\mathcal{U}, \mathcal{F}, \mathcal{C})$. The small number of tasks and deployment sizes are chosen for computational tractability of the optimal solution, and the number of agents is chosen to allow for the cardinality relaxation.

The set of vertices $\mathcal{U}$ is partitioned into two disjoint sets $\mathcal{U}_a$ and $\mathcal{U}_t$, the set of agent nodes (size 18) and task nodes (size 2), respectively. The graph is fully connected, i.e., every node in $\mathcal{U}_a$ is connected to all nodes in $\mathcal{U}_t$. Each edge has a corresponding cost random variable: a truncated Gaussian distribution.\footnote{Algorithm~\ref{a1} does not rely on any Gaussian assumption. Any arbitrary distribution, discrete or continuous, can be used.} These edges represent the distribution of travel times from the agent location to the task location (e.g., in minutes). Means are uniformly sampled from the range $[15,~20]$, standard deviations uniformly sampled from the range $[5,~10]$, and all distributions are truncated at 5 (restricting the minimum travel time between any agent and task to 5). 

Using these parameters, we create 1000 random bipartite graphs and implement Algorithm \ref{a1} using $\alpha = 1 + \log (\max_j J_j(\emptyset) - 1)$ and $\alpha = 1$, then calculate the brute force optimal solution. Figure \ref{fig:algo_v_optimal} shows the aggregate results over the 1000 trials. Because $\alpha = 1 + \log (\max_j J_j(\emptyset) - 1)$ depends on the problem instance, the value of $\alpha$ is not constant among trials and ranges from $[3.8, 4.2]$. The left plots show that the cardinality relaxation yields assignment sizes that are larger than the desired size but costs that are better than optimal. On the right, using $\alpha = 1$ respects the desired deployment size while achieving near-optimal performance empirically. \footnote{The majority of trials using $\alpha=1$ are optimal; only the unshown outliers are sub-optimal, which lie within 10\% of optimal in these simulations.}

\begin{figure}
    \centering
    \text{\footnotesize \quad \quad Benchmark Comparisons of Algorithm 1}\par\medskip
    \includegraphics[scale=0.4]{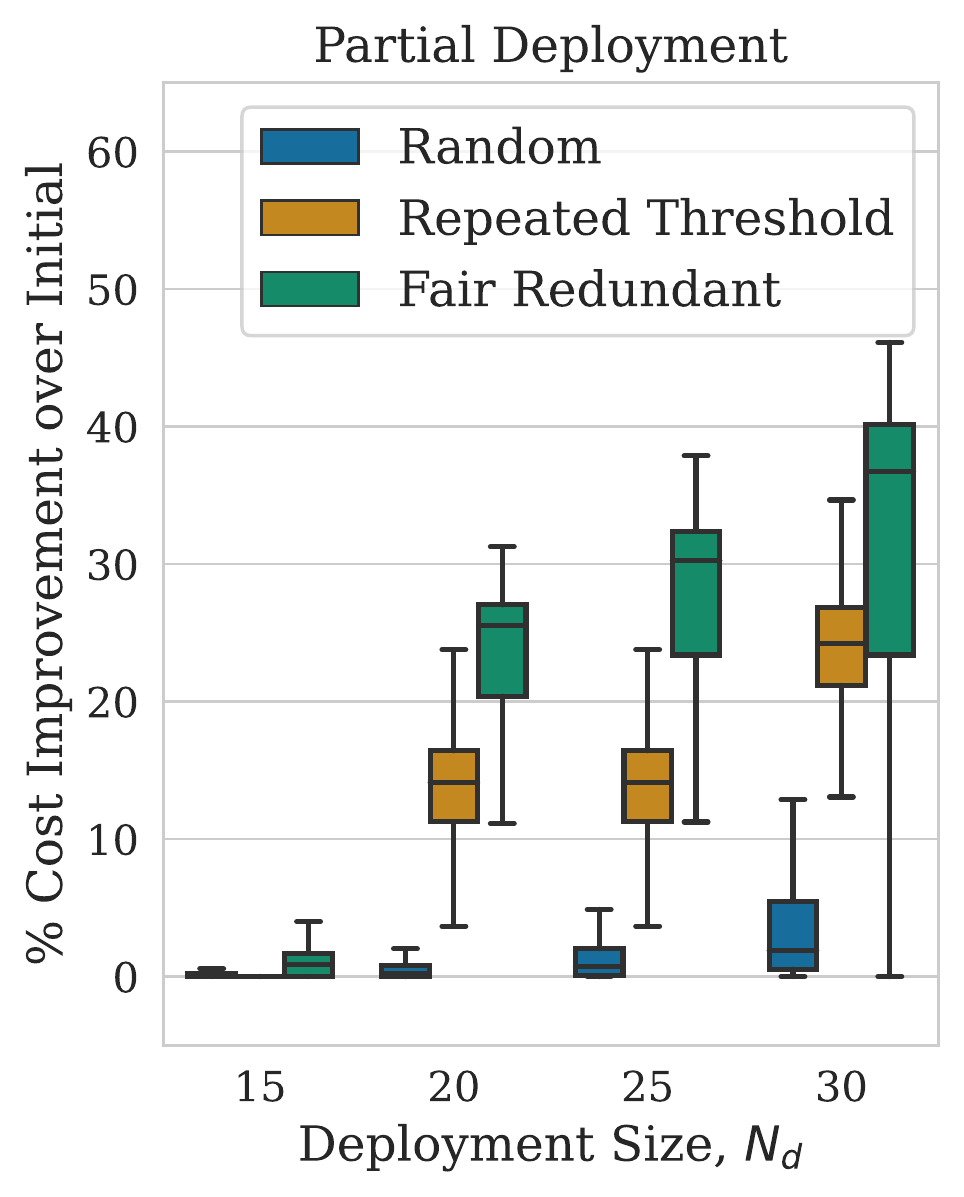}
    \includegraphics[scale=0.4]{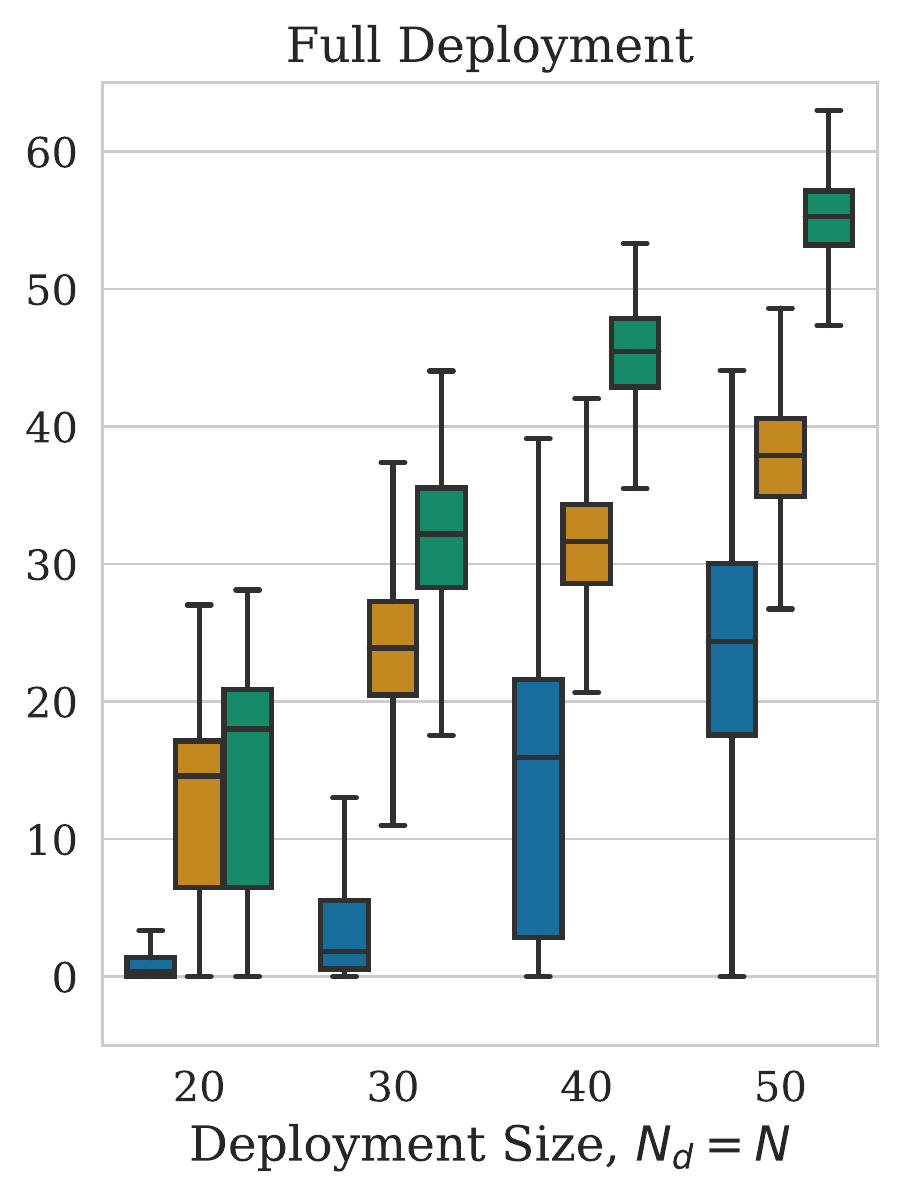}
    \caption{Fair redundant assignment yields the largest improvement in maximum task cost compared to benchmarks. The plots above shows the improvement in cost for random redundant assignment, repeated threshold assignment, and fair redundant assignment. The data show the percent difference between the respective assignment and the initial non-redundant assignment. \textit{Left.} $M=10$ tasks, $N=40$ agents, and varying $N_d$ deployment sizes (\textit{x-axis}). \textit{Right}. $M=10$ tasks with varying number of agents, where the deployment size equals the number of agents, $N=N_d$.
    }
    \label{fig:benchmarks}
\end{figure}

\subsection{Comparing Algorithm \ref{a1} to Benchmarks}

In the following, we show that Algorithm \ref{a1} with $\alpha = 1$ outperforms three benchmarks: \textbf{(1)} non-redundant assignment using the thresholding algorithm, \textbf{(2)} random assignment of redundant agents, \textbf{(3)} repeated iterations of the thresholding algorithm for redundant assignment. We use $\alpha = 1$ since it is shown above to be near optimal without the need for cardinality relaxations. The results in Figure \ref{fig:benchmarks} show that redundant agents improve the performance of non-redundant assignments and that Algorithm~\ref{a1} outperforms all three benchmarks.

We create 1000 random bipartite graphs with the cost random variable parameters listed above. Because calculating the optimal is not needed in these simulations, larger graphs are studied.\footnote{Graphs with tens of agents and tasks are shown in this paper, but the algorithm has been run on examples with hundreds of agents and tasks.} We first consider 40 agents and 10 tasks, with increasing deployment sizes. These trials restrict the deployment size such that some agents are unused. Next, we consider 10 tasks with varying number of agents where all available agents are deployed, i.e., $N = N_d$. These two different problem setups are chosen to show that Algorithm \ref{a1} performs well in scenarios when partial and full deployment is desired. In both cases, Algorithm \ref{a1} outperforms the benchmarks. \footnote{Not shown here is Algorithm \ref{a1} with the $\alpha$ from equation~\eqref{alpha}, which is guaranteed perform at least as good as $\alpha = 1$, but comes at the cost of larger assignments. In the right of Figure \ref{fig:benchmarks}, both $\alpha$ values perform equally because there are no available agents to be used by the cardinality relaxation.}

\subsection{Random Transport Network Case Study}

We instantiate 500 random transport networks with 32 agents and 16 tasks randomly located at different nodes where the travel time along each edge is represented by a Gaussian random variable truncated at zero with mean uniformly sampled from the range $[10,~20]$ and standard deviation uniformly sampled from the range $[5,~10]$. Assignments of size 20 are found using Algorithm~\ref{a1} and the utilitarian approach.\footnote{Both algorithms require an initial non-redundant assignment. To compare the results on the same footing, both algorithms are given the same initial assignment: Hungarian assignment on mean travel times.} 

Figure \ref{fig:task_improvement} shows that the utilitarian approach nearly uniformly helps tasks, while Algorithm \ref{a1} focuses the redundant agents on helping tasks most in need (those with the highest task cost). Additionally, for more than half the trials, the utilitarian approach does not help the worst-off task. All redundant assignment is guaranteed to improve the average task cost because the redundant agents improve the performance of their assigned tasks. While utilitarian redundant assignment better improves the average task cost, it does not necessarily improve the maximum cost. Fair redundant assignment is guaranteed to improve both the maximum and the average task cost.

\begin{figure}
    \centering
    \text{\footnotesize \quad \quad Improvements of Tasks Due to Redundancy}\par\medskip
    \includegraphics[scale=0.42]{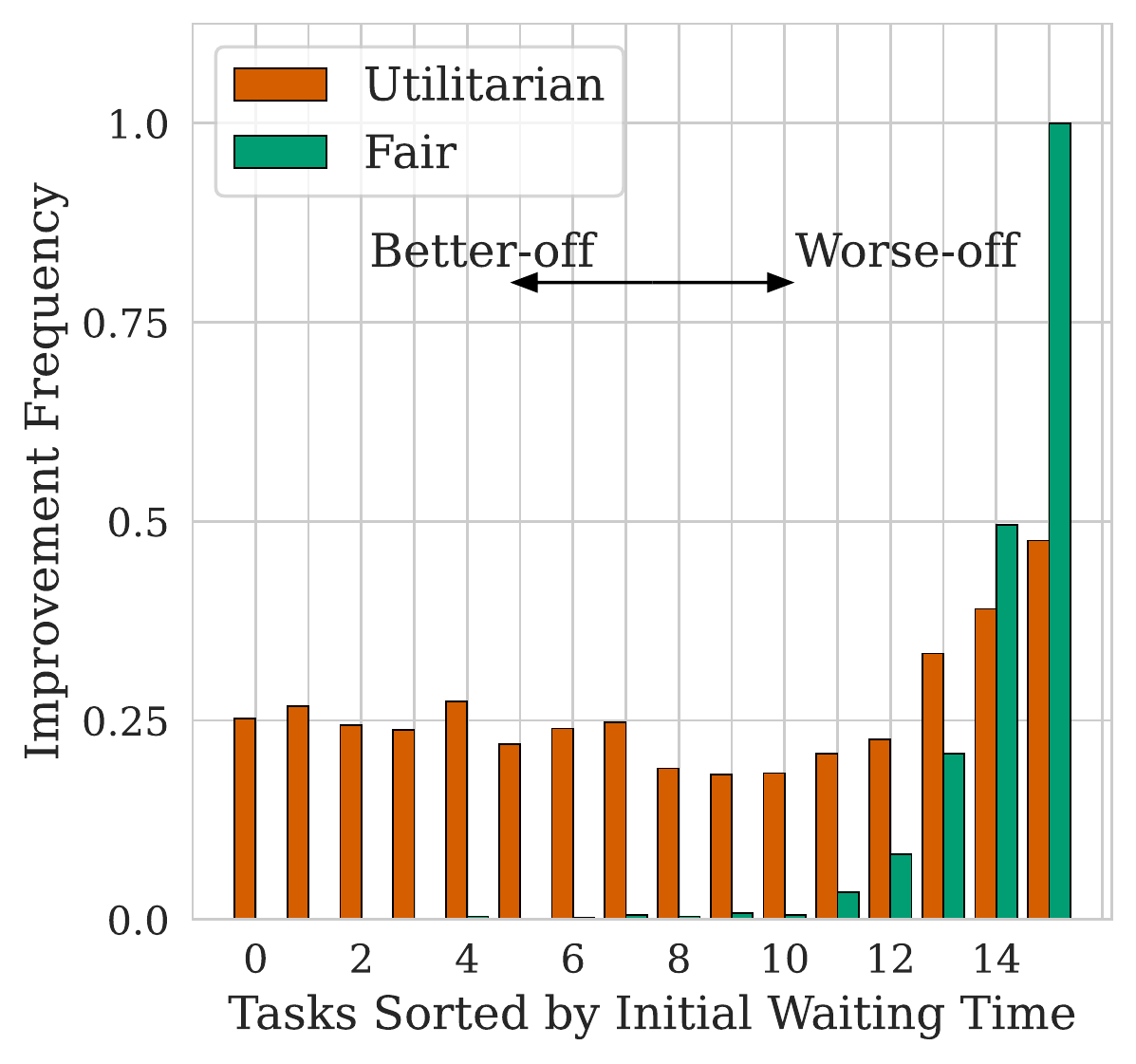}
    \includegraphics[scale=0.42]{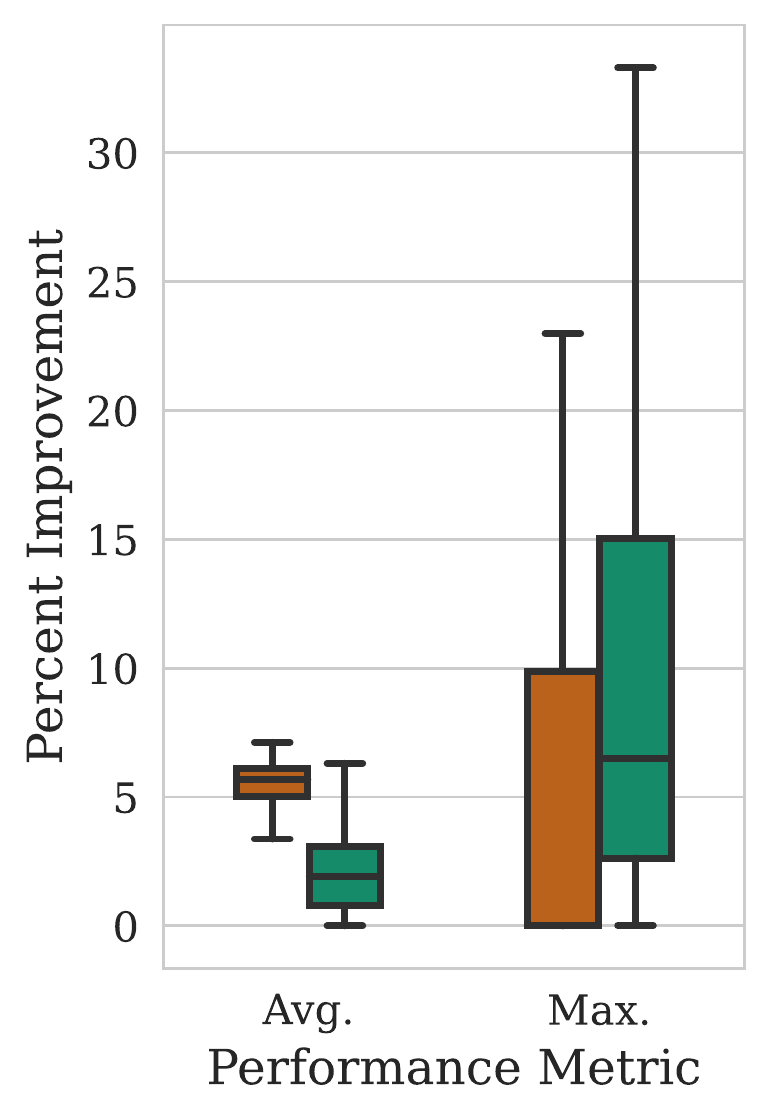}
    \caption{\textit{Left.} The frequency at which different tasks are improved by redundant assignment under the utilitarian and fair algorithms across 500 examples. The tasks are ordered by increasing expected cost after the initial assignment, showing that the fair approach always helps the tasks that are worse off while the greedy approach more uniformly helps tasks. \textit{Right.} The percent improvement in the average and maximum task cost compared to those of the initial assignment shows that the fair assignment does not improve the average as much as does the utilitarian assignment but strongly outperforms in improvement of the maximum cost.
    }
    \label{fig:task_improvement}
\end{figure}

\section{Discussion}

This work is inspired by robotic applications such as rescue and medical or emergency supplies delivery. The objectives of such applications focus on task completion time and omit other costs (e.g., energy used). While previous work focuses on average performance, we observe that the human-centric nature of these applications requires considerations of fairness. While our work contributes to fairness in redundant assignment by improving the objective function, we acknowledge that fairness is more than a mathematical formulation and we welcome discussions of the broader impact of these algorithms. 

For practitioners, note the similarity of our problem formulation and approach to works on providing robustness to a worst-case objective. This formulation is flexible in that it has few assumptions, such as the supermodularity of the cost functions. Thus far, work in redundant assignment has focused on the first-come first-to-serve principle. However, Algorithm~\ref{a1} can be applied to many multi-agent problems where uncertainty impacts task performance, requiring new collaboration functions, $\J$.

When studying other applications, scalability must be considered. Theorem~\ref{T} shows that the complexity of Algorithm~\ref{a1} is dominated by the greedy term, which scales linearly with the dominant term, $N$, because the number of tasks, $M$, and degree of redundancy, $N_d$ are small. In instances where both $N_d$ and $N$ are large, the scaling is approximately quadratic. In such cases, our centralized approach is less tractable, and therefore could be adapted via hierarchical approaches or by developing new decentralized algorithms.

Last, we note that existing approaches in redundant assignment, including this work, assume an initial non-redundant assignment. Therefore, no theoretical guarantees have been proven for the full assignment of initial and redundant agents thus far. Doing so will require new algorithms that do not rely on the initial assignment to maintain supermodularity and that ensure that all tasks are assigned at least one agent.

\section{Conclusions}

Redundant assignment provides robustness against uncertainty and improves task performance. While redundancy comes at the cost of using more robots, it is beneficial in time sensitive scenarios such as rescue or emergency delivery.

Inspired by the human-centric nature of these application spaces (e.g., rescue), we build on previous work in redundant assignment to consider fairness; notably, we are first to formalize fairness in this space. Using a Rawlsian approach to fairness, we formulate fair redundant task allocation as the optimization of worst-case task cost with a cardinality constraint, a problem that is NP-Hard. We exploit the natural supermodularity of the problem to propose a tractable solution. 

Algorithm \ref{a1} is a binary search where the target value is determined by greedily solving the relaxed sub-problem shown in equation 5. This algorithm converges to an approximate solution: an assignment that is $\alpha$ times larger than the cardinality constraint. We prove performance guarantees of this algorithm, showing that the returned approximate solution has a cost that is less than that of the optimal solution but a solution size that is $\alpha$ times larger. Theorem~\ref{T} proves bounds on this $\alpha$ value.

Additionally, we show empirically that Algorithm \ref{a1} with $\alpha=1$ (meaning the cardinality constraint remains satisfied) provides near optimal results, despite having no theoretical guarantees. Simulated experiments show that this approach outperforms benchmarks, scales to large problem instances, and provides both increased fairness and increased utilitarian social welfare over non-redundant assignments.

\appendix
\subsection{Supermodularity}\label{supermod}
Supermodularity implies that adding an element $x$ to a set $\A$ results in a larger or equal marginal decrease than when $x$ is added to a \textbf{superset} of $\A$. This is known as the property of diminishing returns; there are diminishing returns from an added element $x$ as the set it is added to grows larger. Below are the formal definitions of supermodularity \cite{prorok_robust_2020}.

\begin{definition} \label{marginal_decr}
Marginal Decrease. For a finite set $\mathcal{F}$ and a given set function $J: 2^{\mathcal{F}} \mapsto \mathbb{R},$ the marginal decrease of $J$ at a subset $\mathcal{A} \subseteq \mathcal{F}$ with respect to an element $x \in \mathcal{F} \backslash \mathcal{A}$ is:
\begin{equation}
\Delta_{J}(x | \mathcal{A}) \triangleq J(\mathcal{A})-J(\mathcal{A} \cup\{x\})
\end{equation}
\end{definition}

\begin{definition}\label{supermod_def}
Supermodularity. Let $J: 2^{\mathcal{F}} \mapsto \mathbb{R}$ and $\mathcal{A} \subseteq \mathcal{B} \subseteq$ $\mathcal{F} .$ The set function $J$ is supermodular\footnote{Note, a supermodular function $J$ implies $-J$ is submodular.} if and only if for any $x \in \mathcal{F} \backslash \mathcal{B}$
\begin{equation}
\Delta_{J}(x | \mathcal{A}) \geq \Delta_{J}(x | \mathcal{B})
\end{equation}
\end{definition}

\begin{lemma}\label{supermod_lemma}
The maximum of a set of supermodular functions is not supermodular.
\end{lemma}

\begin{proof}
Consider an assignment set $\A$ and an element $x$ to be added. Let $x$ decrease the value of $J_1(\A)$. Here, $x$ is an assignment of a new agent to task 1, thus decreasing the expected cost of task 1. If $J_1$ is not the maximum of the all $J_j$, the addition of $x$ to $\A$ has no effect on $\max_j(\J)$. In other words, its marginal decrease of $\max_j(\J)$ with respect to $x$ is zero. By Definition \ref{supermod_def}, $\max_j(J_j)$ is only supermodular if the marginal decrease of $\max_j(J_j)$ with respect to $x$ is also zero for \textbf{all supersets} of $\A$.

Let $\mathcal{B}$ be a superset of $\mathcal{A}$ that improves upon all task costs except task 1, such that the maximum cost of $\mathcal{B}$ is $J_1$. Here $\mathcal{B} \supseteq \A$ such that $\max_j(J_j(\mathcal{B})) = J_1(\mathcal{B}) = J_1(\A)$.\footnote{This is feasible because the cost distributions are unconstrained.} 
Since the assignment $x$ decreases the value of $J_1(\A)$, it decreases the value of $\max_j(J_j(\mathcal{B}))$. Therefore, the marginal decrease of adding $x$ to $\mathcal{B}$ is nonzero, meaning $\max_j(J_j)$ is not supermodular.
\end{proof}
\footnotesize\printbibliography
\end{document}